\documentclass{article}

\usepackage[a4paper, total={6in, 8in}]{geometry}

\usepackage{hyperref}       
\usepackage{url}            
\usepackage{booktabs}       
\usepackage{amsfonts}       
\usepackage{nicefrac}       
\usepackage{microtype}      
\usepackage{xcolor}         

\usepackage{times}
\usepackage{soul}
\usepackage{url}
\usepackage{graphicx}
\usepackage{amsmath}
\usepackage{amsthm}
\usepackage{booktabs}
\usepackage{algorithm}
\usepackage{algorithmic}
\newtheorem{definition}{Definition}

\newtheorem{theorem}{Theorem}
\newtheorem{proposition}{Proposition}

\newcommand{\pa}{\operatorname{PA}}
\newcommand{\epa}{\operatorname{EPA}}
\newcommand{\pc}{\operatorname{PC}}
\newcommand{\mb}{\operatorname{MB}}

\newcommand{\ce}{\operatorname{CE}}
\newcommand{\ece}{\operatorname{ECE}}
\newcommand{\eece}{\operatorname{EECE}}

\newcommand{\cate}{\operatorname{CATE}}
\newcommand{\ate}{\operatorname{ATE}}

\renewcommand{\emph}{\textit}

\makeatletter
\newcommand*{\indep}{%
	\mathbin{%
		\mathpalette{\@indep}{}%
	}%
}

\newcommand*{\nindep}{%
	\mathbin{
		\mathpalette{\@indep}{/}%
	}%
}

\newcommand*{\@indep}[2]{%
	\sbox0{$#1\perp\m@th$}
	\sbox2{$#1=$}
	\sbox4{$#1\vcenter{}$}
	\rlap{\copy0}
	\dimen@=\dimexpr\ht2-\ht4-.2pt\relax
	\kern\dimen@
	\ifx\\#2\\%
	\else
	\hbox to \wd2{\hss$#1#2\m@th$\hss}%
	\kern-\wd2 %
	\fi
	\kern\dimen@
	\copy0 
}
\makeatother

\title{Explanatory causal effects for model agnostic explanations}

\author{ 
Jiuyong Li$^1$ \and Ha Xuan Tran$^1$ \and Thuc Duy Le$^1$ \and Lin Liu$^1$ \and Kui Yu$^2$ \and Jixue Liu$^1$ \\
$^1$ University of South Australia \\
$^2$ School of Computer and Information, Hefei University of Technology \\
}

\begin{document}

\maketitle

\begin{abstract}
This paper studies the problem of estimating the contributions of features to the prediction of a specific instance by a machine learning model and the overall contribution of a feature to the model. The causal effect of a feature (variable) on the predicted outcome reflects the contribution of the feature to a prediction very well. 
A challenge is that most existing causal effects cannot be estimated from data without a known causal graph. In this paper, we define an explanatory causal effect based on a hypothetical ideal experiment. The definition brings several benefits to model agnostic explanations. First, explanations are transparent and have causal meanings. Second, the explanatory causal effect estimation can be data driven. Third, the causal effects provide both a local explanation for a specific prediction and a global explanation showing the overall importance of a feature in a predictive model. We further propose a method using individual and combined variables based on explanatory causal effects for explanations.  
%
We show the definition and the method work with experiments on some real-world data sets.   
\end{abstract}


\section{Introduction}

Now more and more black box machine learning models are used for various decision making applications, and users are interested in knowing how a decision is made. 

This paper studies the problem of estimating the contributions of features in a specific prediction by a machine learning model and the overall contributions of features to a predictive model. The first part belongs to prediction level (local) post hoc interpretation~\cite{Shapley-ValuePredictionContribution,Baehrens-Explain-Individual-Decisions2010,ribeiro2016should,lundberg2017unified}, and the second part is closely related to model level (global) post hoc interpretation~\cite{AltmannFeatureImportance2010,Olden2004-Variableimportance}. Quite a body of work has been done to solve the problem~\cite{MurdochPNASReview2021,MILLER20191,ExplainingExplanations,Review-InterpretableML2020}. This paper focuses on causality based solutions. 

Interpreting the contributions of features to predictions causally is a natural way of explaining the decisions made by predictive models. An intuitive measure of the contribution of a feature to a prediction is the weight added by the presence of the feature to the prediction, and this weight should represent the pure contribution of the feature (in causal terms, the weight is not biased by confounders). The causal effect of a feature on the predicted outcome indicates the change of the predicted outcome due to a change of (i.e. an intervention on) the feature, and best captures the contribution of a feature to a prediction. 

In the past few years, many causally motivated methods for post hoc interpretation have been developed. such as~\cite{SchwabCXPlain2019,kim2018interpretability,goyal2019explaining,NeuralNetworkAttributions2019,CounterfactualModularDeepGenerativeModels2020,generativeCausalExplanation2020,CausalAttribution2019,NeuralNetworkAttributions2019,harradon2018causal,CausalInterpretation2020}. They share some of the following limitations.  

First, the process of estimating causal effects may not be transparent, and this hinders the understanding of explanations. 
In many solutions, the objective function is causally motivated, but then the causal effect is estimated by a complex optimisation method (in many cases, a black box model). The lack of transparency makes causal interpretation difficult.


Second, the assumptions for the causal interpretation
are unclear in many works, and this leads to ambiguous causal semantics.  
Many works do not discuss assumptions for a causal interpretation. Without assumptions, there might not be a causal meaning. For example
Asymmetric Shapley Values (ASVs)~\cite{AsymmetricShapleyValues2020}  relax the symmetric requirement in Shapley value calculation and incorporate causal directions for asymmetric calculation. However, it is unclear whether ASVs represent causal effects. 



Third, some methods require strong domain knowledge. The method in~\cite{karimi2021algorithmicREcourse} needs a known causal graph and the method in~\cite{CLEAR2020} requires a list of user chosen variables as an input. The knowledge may not be available in many applications. 

Fourth, the majority of methods are specifically for image classification. There needs a general purpose causal explanation measure for other black box models, such as random forest~\cite{Breiman2001_RandomForest} and SVM~\cite{SupporVectorNetworks1995}.      



We aim at defining an explanatory causal effect (ECE) for the general purpose explanation of black box models with the full assumptions as in graphic causal model~\cite{Pearl2009_Book}. The ECE is simple and transparent to make explanations explicit. 

The defined explanatory causal effect needs to be estimated from data with the minimum domain knowledge, but purely data driven causal effect estimation poses a major challenge. A causal graph which encodes causal relationships is necessary to estimate causal effects, but a causal graph is unavailable in most applications, and thus needs to be learned from data. However, it is impossible to learn from a data set the unique causal graph and only an equivalence class of causal graphs, leading to multiple estimated values of a causal effect~\cite{Maathuis-IDA-Statist}. 




We make the following contributions in this paper. 

\begin{enumerate}
\item We define explanatory causal effects (ECEs) based on a hypothetical experiment, and show that the estimation of ECEs does not need a complete causal graph but a local causal graph, which can be discovered in data uniquely in the causal explanation setting. Therefore, ECEs can be estimated from data. 
\item We present a method to explain a prediction and a model with individual and combined variables based on ECEs to enrich explanations. 
\item We use relational, text and image data sets to demonstrate the proposed ECEs and the method work.  
\end{enumerate}


\section{Background}
\label{sec:background}
In this section, we present the necessary background of causal inference. We use upper case letters to represent variables and bold-faced upper case letters to denote sets of variables. The values of variables are represented using lower case letters.

Let $\mathcal{G}=(\mathbf{V},\mathbf{E})$ be a graph, where $\mathbf{V}=\{V_{1},\dots, V_{p}\}$ is the set of nodes and $\mathbf{E}$ is the set of edges between the nodes, i.e. $\mathbf{E}\subseteq \mathbf{V}\times \mathbf{V}$. A path $\pi$ is a sequence of distinct nodes such that every pair of successive nodes are adjacent in $\mathcal{G}$. A path $\pi$  is directed if all arrows of edges along the path point towards the same direction.  A path between $(V_i,  V_j)$ is a backdoor path with respect to $V_i$ if it has an arrow into $V_i$. Given a path $\pi$, $V_{k}$ is a collider node on $\pi$ if there are two edges incident such that  $V_{i}\rightarrow V_{k} \leftarrow V_{j}$. In $\mathcal{G}$, if there exists $V_i\rightarrow V_j$, $V_i$ is a parent of $V_j$ and we use $\pa(V_j)$ to denote the set of all parents of $V_j$. In a directed path $\pi$, $V_i$ is an ancestor of $V_j$ and $V_j$ is a descendant of $V_i$ if all arrows point towards $V_j$. 

A DAG (Directed Acyclic Graph) is a directed graph without directed cycles. With the following two assumptions, a DAG links to a distribution. 

\begin{definition} [Markov condition~\cite{Pearl2009_Book}]
	\label{Markovcondition}
	Given a DAG $\mathcal{G}=(\mathbf{V}, \mathbf{E})$ and $P(\mathbf{V})$, the joint probability distribution of $\mathbf{V}$, $\mathcal{G}$ satisfies the Markov condition if for $\forall V_i \in \mathbf{V}$, $V_i$ is probabilistically independent of all non-descendants of $V_i$, given the parents of $V_i$.
\end{definition}


\begin{definition}[Faithfulness~\cite{spirtes2000causation}]
	\label{Faithfulness}
	A DAG $\mathcal{G}=(\mathbf{V}, \mathbf{E})$ is faithful to $P(\mathbf{V})$ iff every independence presenting in $P(\mathbf{V})$  is entailed by $\mathcal{G}$, which fulfils the Markov condition. A distribution $P(\mathbf{V})$ is faithful to a DAG $\mathcal{G}$ iff there exists DAG $\mathcal{G}$ which is faithful to $P(\mathbf{V})$.
\end{definition}

With the Markov condition and faithfulness assumption~\cite{spirtes2000causation}, we can read the (in)dependencies between variables in $P(V)$ from a DAG using $d$-separation~\cite{Pearl2009_Book}. 

\begin{definition} [$d$-Separation]
	A path in a DAG is $d$-separated by a set of nodes $\mathbf{S}$ if and only if: (1) $\mathbf{S}$ contains the middle node, $V_k$, of a chain $V_i \to V_k \to V_j$, $V_i \leftarrow V_k \leftarrow V_j$, or  $V_i \leftarrow V_k \to V_j$ in path $p$; (2) when path $p$ contains a collider $V_i \to V_k \leftarrow V_j$, node $V_k$ or any descendant of $V_k$ is not in $\mathbf{S}$. 
\end{definition}

A DAG is causal if the edge between two variables is interpreted as a causal relationship. To learn a causal DAG from data, another assumption is needed.  

\begin{definition}[Causal sufficiency~\cite{spirtes2000causation}]
	For every pair of variables observed in a data set, all their common causes are also observed in the data set.
\end{definition}





With the Markov condition, faithfulness assumption and causal sufficiency assumption, a DAG learned from data can be considered as a causal DAG, where a directed edge represents a causal relationship.


The causal (treatment) effect indicates the strength of a causal relationship. To define the causal effect, we introduce the concept of intervention, which forces a variable to take a value, often denoted by a \emph{do} operator in literature~\cite{Pearl2009_Book}. A \emph{do} operation mimics an intervention in a real world experiment. For example, $do(X=1)$ means $X$ is intervened to take value 1. $P(y \mid do (X=1))$ is an interventional probability. 







The direct effect is a measure of treatment effect directly from the interventional variable to the outcome when all other variables are controlled in an ideal experiment. 

\begin{definition}[Direct effect~\cite{Pearl2009_Book}]
	\label{def:CDE}
	The direct effect of $X$ on $Y$ is given by $P(Y \mid do (X=x), do (\mathbf{V}_{\backslash XY} = \mathbf{v}))$ where $\mathbf{V}$ means all variables in the system, and $\mathbf{V}_{\backslash XY}$ indicates all other variables except $X$ and $Y$.
\end{definition}

To infer interventional probabilities (by reducing them to normal conditional probabilities, the set of rules of $do$-calculus~\cite{Pearl2009_Book} and a causal DAG are necessary. 

The local causal structure around a variable, e.g.$V$, provides the most information for predicting $V$.  Markov blanket is a local structure~\cite{PearlProbabilisticBook1988,Aliferis:2010}. 

\begin{definition}[Markov blanket] In a DAG, the Markov blanket of a variable $V$, denoted as, $\mb(V)$,  is unique and consists of its parents, children, and the parents of its children. 
\end{definition}      

The Markov blanket renders all other variables independent of $V$ given its Markov blanket. This is why it is frequently used in feature selection~\cite{YuUnifiedFeatureSelection2021}. Importantly, a Markov blanket can be identified in data. 

Another important local structure is called PC set (Parent and Child)~\cite{Aliferis:2010}, i.e. $\pc(V)$, which includes the parents and children of $V$. When there are no descendants of $V$ in a system, $\pc(V) = \mb(V) = \pa(Y)$. $\pc(V)$ is an important set of variables for predicting and explaining $V$.

\section{Explanatory causal effects}
Let us consider a classifier $Y = f(\mathbf{X})$ to be explained where $Y$ is the predicted outcome and $\mathbf{X}$ is a set of features. $X_i$ for $1 \le i \le m$ and $Y$ are binary. 
%
%
We assume the Markov condition, causal sufficiency and faithfulness are satisfied and we do not repeat the statements in each theorem or definition. In this paper, we use the terms variable and feature exchangeably.

\begin{figure}[tb]
	\begin{center}
			\includegraphics[width=0.70\textwidth]{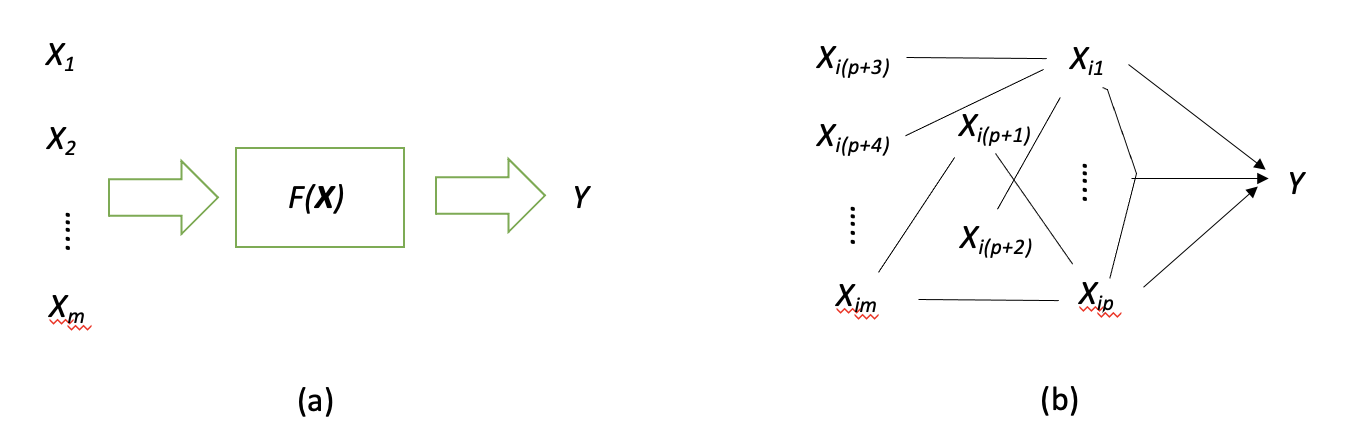}
			\caption{(a) A classifier $Y = f(\mathbf{X})$ built from a black box model. $X_1, X_2, \dots, X_m$ are inputs of $f(\mathbf{X})$, and $Y$ is the output. $X_1, X_2, \dots, X_m$ are considered as potential causes of $Y$ in this system. (b) A causal model of the classifier is represented as a Partially Directed Acyclic Graph (PDAG). Note that the undirected edges mean that their directions are likely uncertain from data since many edges will be left non-orientated when learning a DAG from data. }
			\label{fig_CusalModel}
		\end{center}
\end{figure}

The aim of the work is to explain the classifier $Y = f(\mathbf{X})$. The system under our consideration is shown in Figure~\ref{fig_CusalModel}(a). In this system, $X_1, X_2, \ldots X_m$, i.e. features, are potential causes of $Y$ since they are inputs and $Y$ is the output, i.e. predicted outcome. $Y$ does not have descendants.  Figure~\ref{fig_CusalModel}(b) shows a causal model for the classifier $Y = f(\mathbf{X})$ learned from data with features and predicted outcomes and it is a Partially Directed Acyclic Graph (PDAG) containing non-oriented edges. This is because, from data, only an equivalence class of DAGs can be learned, not a unique DAG. A group of DAGs encoding the same dependencies and dependencies among variables is called an equivalence class of DAGs. For example, DAGs $X_1 \to X_2 \to X_3$, $X_1 \leftarrow X_2 \leftarrow X_3$, and $X_1 \leftarrow X_2 \to X_3$ belong to the same equivalence of class since they encode $X_1 \indep X_3 \mid X_2$ where $\indep$ indicates independence.  In this PDAG, the parents of $Y$ are certain since there are no descendants of $Y$. Given the sufficiency of data and no errors in conditional independence tests, the parents of $Y$ can be identified.

\textbf{Remarks} A causal graph is essential for causal effect estimation, but there is no causal graph in many applications. When learning a causal graph from data, only an equivalence class of causal graphs can be identified.  Due to the uncertainty of the DAG, a number of well known causal effects, such as the Average Treatment Effect (ATE)~\cite{ImbensRubin2015_Book,Pearl2009_Book}, Conditional Average Treatment Effect (CATE)~\cite{abrevaya2015estimating,AtheyImbens2016_PNAS}, natural direct causal effect~\cite{Pearl2009_Book} cannot be estimated, and hence cannot be used for explanations. Counterfactual inference~\cite{Pearl2009_Book} cannot be conducted when there is no unique DAG either. 

To understand the function of $f(\mathbf{X})$, we do the following hypothetically ideal experiment. We change the value of variable $X_i$ while setting values of other variables  $\mathbf{X} \backslash X_i$ to some controlled (fixed) values to observe the change of $Y$. Because this is an ideal experiment, the system is fully manipulable and the values of all variables can be set to whatever values irrespective of the values of other variables. We use Pearl's direct effect to quantify the causal effect of variable $X_i$ when all other variables are controlled to $\mathbf{X' = x'}$ where $\mathbf{X' = X}\backslash \{X_i\}$. 


\begin{definition}[Explanatory causal effect ($\ece$)]
	\label{def:ECE}
	The explanatory causal effect of $X_i$ on $Y$ with the values of other variables in the system set to $\mathbf{x'}$ is $\ece (X_i, Y \mid \mathbf{x'}) = P(y \mid do (X_i=1), do (\mathbf{X' = x'})) - P(y \mid do (X_i=0), do (\mathbf{X' = x'}))$.
\end{definition}

$\ece$ is the effect of the change of a variable on the predicted outcome when all other variables are intervened to some values. The explanatory causal effect can be identified in data as follows in the system.  


Let $\pa'(Y) = \pa(Y) \backslash \{X_i\}$, and $\mathbf{p'}$ be a value of $\pa'(Y)$. Hence, $\mathbf{p' \subseteq x'}$. Let $y$ stand for $Y=1$ for brevity.

\begin{theorem}
	\label{theo:main}
	In the problem setting, the explanatory causal effect  can be estimated as follows:  $\ece (X_i, Y \mid \mathbf{x'}) =  P(y \mid X_i=1, \mathbf{p'}) - P(y \mid X_i=0, \mathbf{p'})$ if $ X_i \in \pa(Y)$; and 
$\ece (X_i, Y, \mathbf{x'}) = 0$ otherwise. 
\end{theorem}
%
The proof is in the Appendix	

Estimating $\ece$ does not need a complete causal graph, but a local causal graph, i.e. the parents of $Y$. In our problem setting, $\ece$ can be estimated uniquely from data since $\pa(Y) = \pc(Y)$ can be uncovered in data. The complexity for finding a complete causal graph is super-exponential to the number of variables~\cite{Chickering:2004}, but the complexity for finding a local causal structure is polynomial to the number of variables~\cite{PC-select-2010,Tsamardinos2006_MMPC,Aliferis2003_Hiton}. A local causal structure can be found more efficiently and reliably in data than a global causal structure. It has been shown that a PC set can be found in a data set with a few thousands of variables~\cite{PC-select-2010}. Therefore, the explanations based on $\ece$ apply to high dimensional data sets. 



\textbf{Remark:} The direct effect in Theorems 4.5.3 and 4.4.6~\cite{Pearl2009_Book} needs an ordered sequence of value assignments, i.e. $do()$, because it is with the semi-Markovian model which allows latent common causes  Therefore, it cannot be used for explanations since its estimation needs a known DAG with latent variables. We assume the causal sufficiency and do not consider latent common causes in the system. We have shown that, when the causal sufficiency is satisfied, for any ordered sequence of value assignments among $\pa(Y)$, $P(y \mid do(\pa(Y)))$ is identifiable in the Appendix. This means that in our case, the value assignments do not need to be in a certain ordered sequence, and hence we can alternate $X_i \in \pa(Y)$ as the treatment variable and assign values to other variables as a controlled environment.

The $\ece$ represents the effect of a variable on the predicted outcome by manipulating the variable while all other variables in the system are set to certain values. The $\ece$ is used for prediction level (local) explanation, i.e. indicating the contribution of a single variable to the prediction when all other variables are fixed to specific values. 

When the contributions of a variable in all controlled environments are aggregated, the average explanatory effect can be used for model level (global) explanation, i.e. on average, the contribution that a variable provides to a predictive model. 

\begin{definition} [Average explanatory causal effect ($\overline{\ece}$)]
	\label{def_aece} 
	Let $\pa(Y)$ include all direct causes of $Y$ and $X_i \in \pa(Y)$, the average explanatory causal effect  $\overline{\ece}(X_i, Y) = \sum_{\mathbf{p'}}(P(y \mid X_i = 1, \pa'(Y) = \mathbf{p'}) - P(y \mid X_i = 0, \pa'(Y) = \mathbf{p'}))P(\mathbf{p'})$. For $X_j \notin \pa(Y)$,  $\overline{\ece}(X_j, Y) = 0$.  
\end{definition}	

$\overline{\ece}$ can be identified in data uniquely since $\overline{\ece}$ is aggregated over values of $\pa'(Y)$. 

\textbf{Remark:} The above discussions reminisce the Structural Equation Model (SEM) where $Y = f(\pa(Y), \epsilon)$ and $\epsilon$ represents errors that are independent of all observed variables in the system. Assuming that the data is generated from a DAG representing the causal mechanism of a system. When the parents of $Y$ can be recovered from the data (or are known), the system, i.e. $f(\cdot)$, is linear and errors are multivariate normal, the direct causal effect (the change of $Y$ caused by a change of a parent of $Y$, e.g. $X_i$ from $x$ to $x+1$) is represented as the regression coefficient $\beta_i$ of $X_i$. When the system is non-linear, such as in the binary case, the above conclusion does not hold~\cite{Pearl2009_Book}. A non-linear model, such as a deep learning model, can be built for $f(\pa(Y), \epsilon)$, but the model loses transparency for explanations. In a binary case, logistic regression may be used to model  $Y = f(\pa(Y), \epsilon)$ regardless of its causal interpretation. The regression coefficients give the overall feature importance but do not provide a local explanation when other features $\mathbf{X} \backslash X_i$ take certain values. 



\begin{definition} [Explanatory causes]
	\label{def_explanatorycauses} 
	A variable $X_i$ is called an explanatory cause if $|\overline{\ece}(X_i, Y)| \ge \epsilon$ where $\epsilon$ is a small positive number.  
\end{definition}	

An explanatory cause is a parent of $Y$ and has a non-zero average explanatory causal effect on $Y$. Explanatory causes are used to explain a model. To explain the prediction of an instance, $\pa(Y)$ is used since an $\ece$ may not follow the average $\overline{\ece}$. 

In the rest of paper, causes are explanatory causes.

\section{Combined causes and interactions}

The accuracy of machine learning models is mostly due to using combinations of multiple variables. We need to consider combined variables for highly accurate explanations.




\begin{definition}[Combined variables]
	Binary variable $W_\mathbf{q}= (X_{i_1}, X_{i_2}, \ldots, X_{i_k})_{\mathbf{q}}$, where $\mathbf{q}$ is a value vector of $W = (X_{i_1}, X_{i_2}, \ldots, X_{i_k})$, is a combined variable if $P((X_{i_1}, X_{i_2}, \ldots, X_{i_k}) = \mathbf{q}) > \sigma$, where $\sigma$ is the minimum frequency requirement. $W_\mathbf{q}=1$ when $(X_{i_1}, X_{i_2}, \ldots, X_{i_k}) = \mathbf{q}$ and $W_\mathbf{q}=0$ otherwise. $V_\mathbf{r}$ is a component variable of $W_\mathbf{q}$ if $V \subset W$ and $\mathbf{r} \subset \mathbf{q}$.  
\end{definition}

The minimum frequency requirement ensures that a combined variable has a certain explanatory power. 

We consider two types of combined variables. The first type aims to combine non-parents of $Y$ to be a cause of $Y$. Non-parents are non-causes, and the purpose is to recover hidden combined causes among non-causes.  





\begin{definition}[Combined causes]
	Let $W_\mathbf{q}$ be a combined variable consisting of non-parents of $Y$. $W_\mathbf{q}$ is a combined cause of $Y$ if $|\overline{\ece}(W_\mathbf{q}, Y)| \ge 0$ where  $\overline{\ece}(W_\mathbf{q}, Y) = \sum_{\mathbf{p}} ((P(y \mid W_\mathbf{q}=1, \pa(Y)= \mathbf{p}) - P(y \mid W_\mathbf{q} = 0, \pa(Y)=\mathbf{p}))P(\mathbf{p})$.  
\end{definition}



The second type includes at least one parent of $Y$. The parent(s) interact with (each other and) other variables to make the average direct causal effect stronger than without interaction. 

\begin{definition}[Interactions]
	Let $W_\mathbf{q}$ be a combined variable including at least one parent of $Y$.  Let $X_j$ be the parent with the highest $\overline{\ece}$ among all parents in the combined variable. $W_\mathbf{q}$ represents an interaction if $|\overline{\ece}(W_\mathbf{q}, Y)| > |\overline{\ece}(X_j, Y)|$ where $\overline{\ece}(W_\mathbf{q}, Y) = \sum_{\mathbf{p'}} (P(y \mid W_\mathbf{q}=1, \pa'(Y)= \mathbf{p'}) - P(y \mid W_\mathbf{q} = 0, \pa'(Y)=\mathbf{p'}))P(\mathbf{p'})$ and $\pa'(Y) = \pa(Y) \backslash (\pa(Y) \cap W)$.      
\end{definition}






\begin{algorithm}[t]
	\caption{Explanatory Causal Effect based Interpretation (ECEI)} 	\label{alg_Explanation}
	{\textbf Input}: Data set $\mathbf{D}$ of $\mathbf{X}=\{X_1, X_2, \ldots, X_m\}$ and $Y$ containing predictions of a black box model to be explained. The maximum length of combined variables $k$, the minimum support $\alpha$, a parameter for non-zero test $\epsilon$, and an input $\mathbf{x}$ to be explained. \\
	{\textbf Output}: Ranked variables for global and local explanations
	\begin{algorithmic}[1]
		\STATE {call a PC algorithm to find $\pc(Y)$ and let $\pa(Y) = \pc(Y)$}\STATE {estimate $\overline{\ece}$ of all parents in $\pa(Y)$}
		\STATE {find closed frequent patterns wrt $\alpha$ up to the length of $k$}
		\STATE {identify combined causes and interactions}
		\STATE {let $\epa(Y)$ contain $\pa(Y)$, combined causes and interactions}
		\STATE {calculate $\overline{\eece}$ of each parent in $\epa(Y)$}
		\STATE {remove redundant combined variables in $\epa(Y)$}		
		\STATE {output $\epa(Y)$ with variables ranked by $\overline{\eece}$  as a global explanations of the model}
		\STATE {calculate $\eece$ of each parent in $\epa(Y)$ based on $\mathbf{x}$}
		\STATE {output $\epa(Y)$ with variables ranked by $\eece$ for local explanation of $\mathbf{x}$}
	\end{algorithmic}
\end{algorithm}

\section{ECEI algorithm \& implementation} 

We will use the parents of $Y$ (including causes), combined causes and interactions for explanations.  We will update $\ece$ and $\overline{\ece}$ with combined causes and interactions. Let $\epa$ be the Extended Parent Set of $Y$ containing all variables in $\pa(Y)$, combined causes of $Y$ and interactions.

\subsection{Local explanation}

For an instance $\mathbf{x}$, the local explanation gives a rank of the extended parents based on their extended $\ece$s, which are defined in the following. 

\begin{definition}[Extended $\ece$]
	\label{def:ECDE}
	Given an instance $\mathbf{x}$. For each $X_i \in \epa(Y)$, the explanatory causal effect of $X_i$ on $Y$ on $\mathbf{x}$ is $\eece(X_i, Y \mid \mathbf{x}) =  P (y \mid X_i = 1, \mathbf{P'=p'}) -  P (y \mid X_i = 0, \mathbf{P'=p'})$ where $\mathbf{P'} = \epa(Y) \backslash (X_i \cup \mathbf{S})$ where $\mathbf{S}$ is the set of variables in $\epa(Y)$, each of which contains $X_i$ or any part of $X_i$ (when $X_i$ is an interaction or a combined direct cause). $\mathbf{p}'$ matches the values in $\mathbf{x}$.
\end{definition} 

$\eece(X_i, Y, \mathbf{x})$ is the causal effect specific to the instance $\mathbf{x}$. 


\subsection{Global explanation}

The global explanation provides a rank of the extended parents based on their extended $\overline{\ece}$s, which are defined in the following. 

\begin{definition}[Extended $\overline{\ece}$]
	\label{def_EACDE}
	For  $X_i \in \epa(Y)$, $\overline{\eece}(X_i, Y) = \sum_{\mathbf{p'}}(P(y \mid X_i = 1, \mathbf{P'=p'}) - P(y \mid X_i = 0, \mathbf{P'=p'}))P(\mathbf{p'})$ where  $\mathbf{P'}$ includes all variables in $\epa(Y)$ which are associated with $X_i$ but excluding all variables containing $X_i$ or any part of $X_i$ (when $X_i$ is an interaction or a combined cause).   
\end{definition}	

The difference between the above definition with the definition of $\overline{\ece}$ is at the constitution of $\mathbf{P'}$. Firstly, an interaction and its components form alternative explanations and we do not consider both simultaneously. Hence, we only consider them separately for calculating $\overline{\eece}$. Secondly, when two variables are not associated, all paths between them are $d$-separated (or there is no path between them). Hence, other parents that are not associated with $X_i$ are omitted in the conditioning set.



\subsection{Implementation}    

The proposed algorithm for ECE based explanation is listed in Algorithm~\ref{alg_Explanation}.  Detailed explanations are provided in the Appendix.

\begin{table}
	\scriptsize
	\center
	\begin{tabular}{clc}
		\multicolumn{3}{c}{ECEI (average for a RF model)}	\\
		\hline   ID & Features  & $\overline{\eece}$  \\
		\hline   1 & Married  & $0.382$   \\  
		2 &		Education.num.12  & $0.322$  \\
		3 & 		Agelt30 & $-0.262$ \\
		4 & 		Prof  & $0.231$  \\
		5 & Education.num.9  & $-0.206$ \\
		& $\dots$ & \\
		8 & \{Self$\_$emp, Male, US\} = \{0, 0, 1\} &  $0.187$ \\
		\hline 
	\end{tabular}
	\begin{tabular}{cc}
		Logistic regression coefficients  & Importances by permutation \\
		\includegraphics[width=0.23 \textwidth]{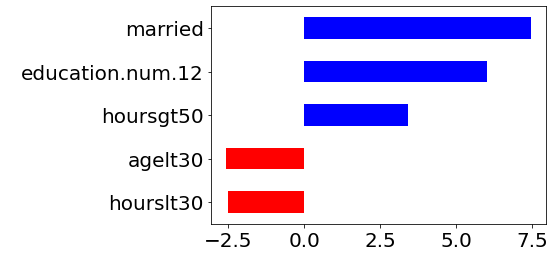} & \raisebox{0\height}{	\includegraphics[width=0.23 \textwidth]{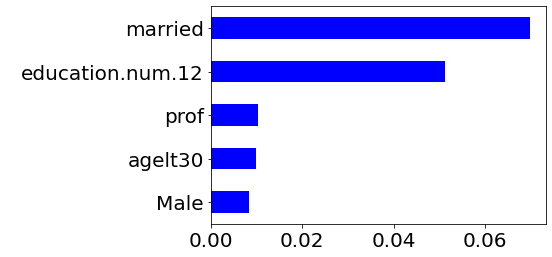}} 
	\end{tabular}
	\caption{Global feature importance in the Adult data set. 
	}
	\label{tab_AdultGlobalResults}
\end{table}

\begin{table}
	\scriptsize
	\center
	\begin{tabular}{clc}
		\multicolumn{3}{c}{ECEI (average for a RF model)}	\\
		\hline   ID & Features  & $\overline{\eece}$  \\
		\hline   
		1 &		Status.1 & $-0.255$ \\
		2 & Credit.1 & $0.184$ \\
		3 & Duration.1 & $0.170$ \\
		4 & Housing.2 & $0.168$ \\
		5 & Credit.history.5 & $0.162$ \\
		& $\dots$ & \\
		10 & \{Debtors.1,Job.3=1,Foreign.1\}=\{1,1,1\} & $0.029$ \\
		\hline 
	\end{tabular}
	\begin{tabular}{cc}
		Logistic regression coefficients  & Importances by permutation \\
		\includegraphics[width=0.23 \textwidth]{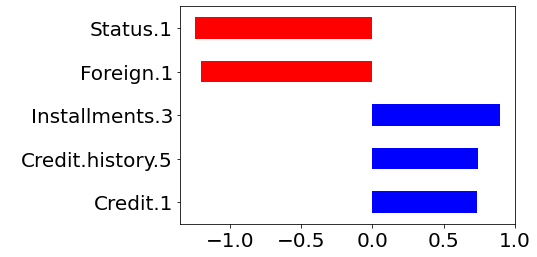} & \raisebox{0\height}{	\includegraphics[width=0.23 \textwidth]{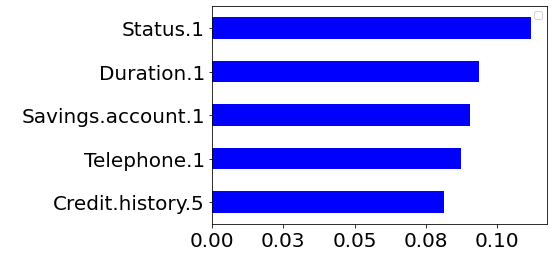}} 
	\end{tabular}
	\caption{Global feature importance in the German credit data set. 
	}
	\label{tab_GermanGlobalResults}
\end{table}

\section{Experiments}

The experiments aim at showing how the proposed ECEs are used for explanations. Since simplicity and transparency are characteristics of  ECEs, we compare ECEs with four well known simple and transparent measures for feature importance:  logistic regression coefficients and feature importance by permutation~\cite{molnar2019}  for global feature importance; and LIME scores~\cite{ribeiro2016should} and Shapley values~\cite{Shapley-ValuePredictionContribution} for the local explanation. See Appendix for other experiment details. 

The evaluation of explanations is very challenging. In experiments, we contrast ECE based explanations with explanations using other measures. We expect certain consistencies between them since all measures should capture some common relationships. We then explain why the differences resulted from our method are reasonable based on our understanding and they are desirable in explanations.

\begin{table}
	\scriptsize
	\center
	\begin{tabular}{clc}
		\multicolumn{3}{c}{ECEI (average for a RF model)}	\\
		\hline   ID & Features  & $\overline{\eece}$  \\
		\hline   
		1 & Odor.7 & $ 0.796$ \\
		2 & Gill.size.2 & $-0.586$ \\
		3 & Bruises.2  & $-0.509$ \\
		4 & Population.5 & $-0.443$ \\
		5 & \{Stalk.surface.below.ring.4,  &  \\
           & Ring.number.2\} = \{1, 1\} & $0.346$ \\
		\hline 
	\end{tabular}
	\begin{tabular}{cc}
		Logistic regression coefficients  & Importances by permutation \\
		\includegraphics[width=0.23 \textwidth]{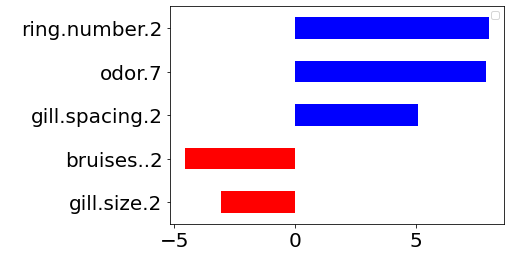} & \raisebox{0\height}{	\includegraphics[width=0.23 \textwidth]{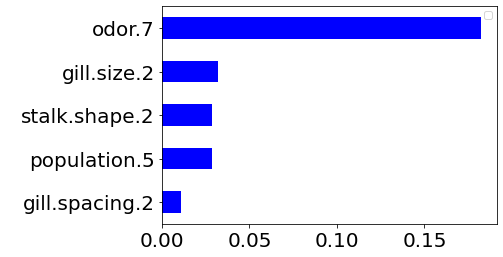}} 
	\end{tabular}
	\caption{Global feature importance in the Mushroom data set. 
	}
	\label{tab_MushroomGlobalResults}
\end{table}

\begin{table}[t]
	\scriptsize
	\center
	\begin{tabular}{lc}
		\multicolumn{2}{c}{ECEI (for an instance prediction by RF)}	\\
		\hline   Feature value contributes to ($\to$)  the predicted class & $\eece$  \\
		\hline   Married = 0 $\to$ Class = 0 (low income)    & $0.307$   \\  
		Education.num.12 = 0 $\to$ Class = 0 & $0.097$  \\
		Prof = 0 $\to$ Class = 0  & $0.061$  \\
		hoursgt50=0  $\to$ Class = 0   & $0.042$ \\
		agelt30=0 $\to$ Class = 0 & $0.028$ \\ 
		\hline 
	\end{tabular}
	\begin{tabular}{cc}
		LIME scores  & Shapley values	\\
		\includegraphics[width=0.15 \textwidth]{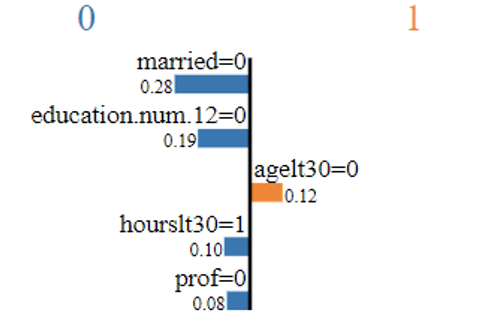} & \raisebox{-0.15\height}{	\includegraphics[width=0.25 \textwidth]{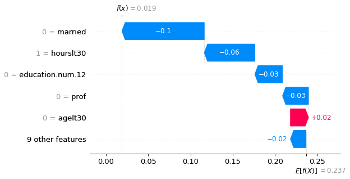}} 
	\end{tabular}
	\caption{Local explanation of the prediction on an instance in the Adult data set. 
	}
	\label{tab_AdultlocalResults}
\end{table}

\begin{table}[t]
	\scriptsize
	\center
	\begin{tabular}{lc}
		\multicolumn{2}{c}{ECEI (for an instance prediction by RF)}	\\
		\hline   Feature value contributes to $\to$  the predicted class & $\eece$  \\
		\hline		Status.1 = 0 $\to$ Class = 1 (good credit) & $0.389$ \\  
		Housing.2 = 1 $\to$ C=1 (short for Class = 1)  & $0.262$   \\
		\{Debtors.1, Job.3, Foreign.1\} = \{1, 1, 1\} $\to$ C=1 & $0.152$ \\
		\{Job.3, Telephone.1, Foreign.1\} = \{1, 1, 1\} $\to$ C=1 & $-0.131$ \\
		Credit.history.5 = 0 $\to$ C=1  & $-0.130$ \\
		\hline 
	\end{tabular}
	\begin{tabular}{cc}
		LIME scores  & Shapley values	\\
		\includegraphics[width=0.19 \textwidth]{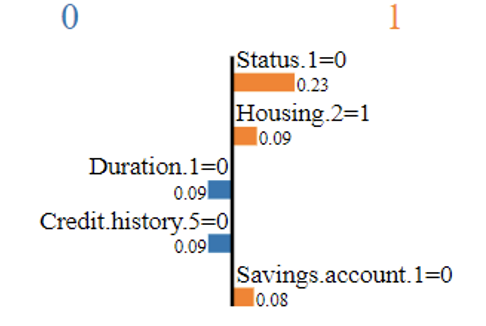} & \raisebox{-0.0\height}{	\includegraphics[width=0.25 \textwidth]{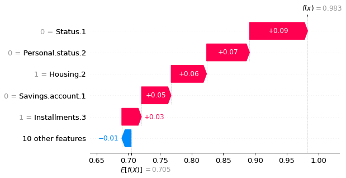}} 
	\end{tabular}
	\caption{Local explanation of the prediction on an instance in the German credit data set. 
	}
	\label{tab_GermanLocalResults}
\end{table}

\begin{table}[b]
	\scriptsize
	\center
	\begin{tabular}{lc}
		\multicolumn{2}{c}{ECEI (for an instance prediction by RF)}	\\
		\hline   Feature value contributes to ($\to$)  the predicted class  & $\eece$  \\
		\hline		
		Bruises.2 = 0 $\to$ Class = 1 (edible) &  $0.893$ \\
		Gill.size.2 = 0 $\to$  Class = 1 & $0.794$ \\
		Stalk.color.above.ring.8 = 1 $\to$ Class=1 & $0.783$ \\
		Population.5 = 0 $\to$ Class = 1 & $0.757$ \\
		\{Stalk.shape.2,stalk.surface.below.ring.4\}=\{0,1\}
		&  \\
		$\to$ Class=1 &  $0.440$ \\
		
		\hline 
	\end{tabular}
	\begin{tabular}{cc}
		LIME scores  & Shapley values	\\
		\includegraphics[width=0.20 \textwidth]{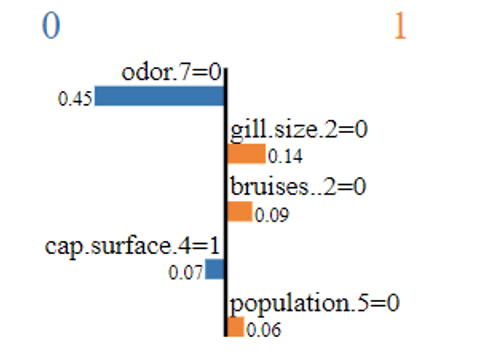} & \raisebox{-0.0\height} {\includegraphics[width=0.25\textwidth]{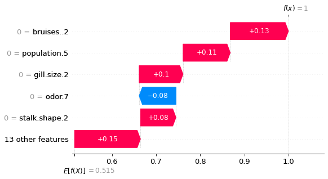}} 
	\end{tabular}
	\caption{Local explanation of the prediction on an instance in the Mushroom data set. 
	}
	\label{tab_MushroomLocalResults}
\end{table}

\subsection{Global explanations}
Global explanations for random forest models on the three data sets are shown in Tables~\ref{tab_AdultGlobalResults}, \ref{tab_GermanGlobalResults} and \ref{tab_MushroomGlobalResults}. 

The explanations of ECEs are meaningful. In the Adult data set,  Education.num.12 ( $>12$ means college education) and Education.num.9 ($< 9$ means not completed high school) are important factors affecting the salary. Marriage status indicates stable jobs and good salaries (note the census data were collected in 1994). Professional occupation is also a good indicator of salary. Agele30 (age $<30$) means that inexperienced employees mostly receive lower salaries. The combined variables (American females who did not work as self employers) indicate a cohort of American citizens who have jobs in the government and private sectors and their salaries are likely high. 

In the German credit data set, Status.1 (balance in the existing checking account $< 0$), Credit.1 (the credit amount $<$ median), Duration.1 (the duration $<$ median), Housing.2 (owning a house), and Credit.history.5 (critical account/other credits existing) are good indicators of good/bad credit.  \{Debtors.1, Job.3, Foreign.1\} = \{1, 1, 1\}  means ``no other debtors/guarantors", ``skilled employees/officials", ``and foreign workers" and this group of people mostly have good credit. 

In the Mushroom data set, Odor.7 (no odor), Gill.size.2 (narrow), Bruises.2 (no bruise), Population.5 (several) and \{Stalk.surface.below.ring, Ring.number\} = \{1, 1\} are indicators for edible or poisonous mushrooms. 

The explanations of ECEs by single variables are quite consistent with those of logistic regression coefficients and feature importance by permutation. In the Adult data set, 3/5 of the top features are consistent with those identified by logistic regression coefficients and 4/5 are consistent with those identified by permutation. In the German Credit data set, 3/5 of the top features are consistent with those identified by both logistic regression coefficients permutation. In the Mushroom data set,  3/4 features of single variables are consistent with those identified by both logistic regression coefficients and permutation. 

The combined variables provide insightful explanations. In the Adult data set, feature importance indicates that men have a certain advantage in earning a high salary over women (an unfortunate truth in the 1990s). Our approach identifies a subgroup of females who are likely to earn a high salary, i.e American females for governments or private companies. In the German Credit data set, Foreign.1 (foreign workers) indicates a negative impact on good credit. Our approach identifies a subgroup of foreign workers likely to have good credit, i.e. skilled/official foreign workers with no other debtors. In the Mushroom data set, stalk.surface.below.ring is smoothly and ring.number = 1 is a good indicator of edible mushrooms by our approach.

\subsection{Local explanations}

Local explanations for the prediction made by the random forest models on an instance in each of the three data sets are shown in Tables~\ref{tab_AdultlocalResults}, \ref{tab_GermanLocalResults} and \ref{tab_MushroomLocalResults}.

The explanations for a prediction of an instance, which may not be consistent with the global explanations, will be determined by the values of the instance. For explaining the prediction for an instance in the Adult data set, Table~\ref{tab_AdultlocalResults} shows the top five contributors for the prediction are from the globally ranked features (\#1, \#2, \#4, \#7, \#6). All five contributors provide consistent explanations of the predicted class.

For the prediction for an instance in the German credit data set, Table~\ref{tab_GermanLocalResults} shows the top five contributors for the prediction are from the globally ranked features (\#1, \#4, \#10, \#9, \#5). The first three contributors provide consistent explanations of the predicted class. Note that,  in this instance, the combined variable \{Debtors.1=1,Job.3=1,Foreign.1=1\} becomes a strong support for the prediction even though its global importance is not high.


For the prediction for the instance in the Mushroom data set, Table~\ref{tab_MushroomLocalResults} shows the top five contributors for the prediction are from the globally ranked features (\#3, \#2, \#10, \#4, \#9). All provide consistent explanations of the predicted class.  Note in the LIME score, Odor.7 = 0 provides a strong but inconsistent explanation of the prediction. In contrast, Odor.7 = 0 has not shown in our top 5 list.


The explanations of ECEs by single variables are quite consistent with those of LIME scores and Shapley values. In the Adult data set, 4/5 features of single variables are shared by three explainers. In German Credit data set, 3/3 features of single variables are shared by explanations of LIME score and 2/3 are shared by explanations of Shapley values. In the Mushroom data set,  3/4 of single variables are shared by three explainers. 

\begin{table}[t]
	\scriptsize
	\center
	\begin{tabular}{c}
		Top 5 superpixels of ECEI   	\\
		\includegraphics[width=0.44 \textwidth]{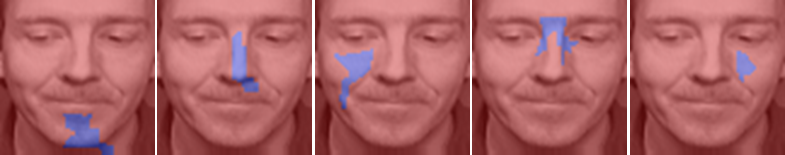} 	
	\end{tabular}
	\begin{tabular}{c}
		Top 5 superpixels of LIME \\
		\includegraphics[width=0.44 \textwidth]{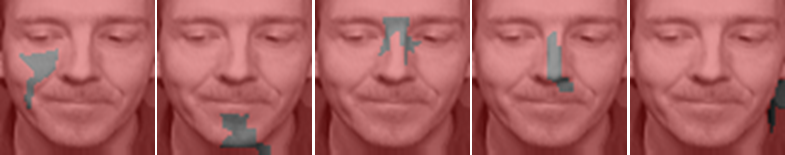} 	\end{tabular}
	\caption{Explanations for an image classified by a random forest model. The top 4 superpixels of ECEI are consistent with LIME's. The fifth superpixel found by ECEI  is meaningful.}
	\label{tab_PictureResults}
\end{table}

\begin{table}[t]
	\scriptsize
	\center
	\begin{tabular}{lc}
		\multicolumn{2}{c}{ECEI (for a text document classified by RF )}	\\
		\hline   Feature values towards prediction  & $\eece$  \\
		\hline		
		{rutgers} $\to$ christian & $0.0326$  \\
		\{jesus, god\} $\to$ christian & $0.0287$ \\
		\{indiana, god\} $\to$ christian & $0.0273$ \\
		\{indiana, jesus\} $\to$ christian & $0.0267$ \\
		{athos} $\to$ christian & $0.0266$ \\
		\hline 
	\end{tabular}
	\begin{tabular}{c}
		LIME scores  	\\
		\includegraphics[width=0.20 \textwidth]{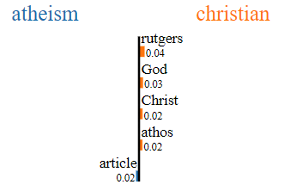} 	\end{tabular}
	\caption{Explanations for text classification. The explanations of ECEI are quite consistent with those of LIME.  
	}
	\label{tab_TextResults}
\end{table}

\subsection{Explaining text and image classifiers}

When texts and images are converted to interpretable representations~\cite{ribeiro2016should}, such as binary vectors indicating the presence/absence of words, or super-pixels (contiguous patches of similar pixels), the proposed ECEI can be used for text and image classification explanations. 

Explanations for an image by a random forest model are shown in  Table~\ref{tab_PictureResults}. ECEI uses one more meaningful superpixel than LIME.

Explanations for a text document classified by a random forest model are shown in Table~\ref{tab_TextResults}. The explanations by ECEI are reasonable and consistent with those provided by LIME.

\section{Conclusion}

In this paper, 
We have defined the explanatory causal effect (ECE) based on a hypothetical ideal experiment, which can be estimated from data without a known causal graph. 
We have proposed an ECE based method, ECEI, for both local and global explanations.  We have used real world data sets to show that the ECEI provides meaningful explanations.
The strengths of explanatory causal effect are that it is data driven, causally interpretable, transparent, and suitable for both local and global explanations. The limitations of the work are that the explanation method does not consider latent variables and the noises of variables.  


\bibliographystyle{abbrv}
\bibliography{../../../CausalHeterogenerityJan2021, ../../../Interpretability, ../../../reference, mybibliography, ../../../causality}

\newpage

\appendix

\section{Proofs of Theorem~1 and Proposition~1}

In this section, we provide proofs for Theorem~1 and Proposition~1. 

Firstly, we introduce some notions used in the proofs. For a DAG $\mathcal{G}$ and  subsets of nodes $\mathbf{X}$ and $\mathbf{Z}$ where $\mathbf{X} \cap \mathbf{Z} = \emptyset$ in $\mathcal{G}$, $\mathcal{G}_{\overline{\mathbf{X}}}$ represents the DAG obtained by deleting from $\mathcal{G}$ all incoming edges to nodes in $\mathbf{X}$, $\mathcal{G}_{\underline{\mathbf{X}}}$ the DAG by deleting from $\mathcal{G}$ all outgoing edges from nodes in $\mathbf{X}$, and $\mathcal{G}_{\overline{\mathbf{X}} \underline{\mathbf{Z}}}$ the DAG by deleting from $\mathcal{G}$ all incoming edges to nodes in $\mathbf{X}$ and outgoing edges from nodes in $\mathbf{Z}$. DAGs $\mathcal{G}_{\overline{\mathbf{X}}}$, $\mathcal{G}_{\underline{\mathbf{X}}}$ and $\mathcal{G}_{\overline{\mathbf{X}} \underline{\mathbf{Z}}}$ are manipulated DAGs and the conditional independences among variables can be read off from the manipulated DAGs.

\begin{theorem}
	\label{theo:main}
	In the problem setting, the explanatory causal effect  can be estimated as follows:  $\ece (X_i, Y \mid \mathbf{x'}) =  P(y \mid X_i=1, \mathbf{p'}) - P(y \mid X_i=0, \mathbf{p'})$ if $ X_i \in \pa(Y)$; and 
	$\ece (X_i, Y \mid \mathbf{x'}) = 0$ otherwise. 
\end{theorem}
\begin{proof}
	We first recap the problem setting. We assume that Markov condition, causal sufficiency and faithfulness are satisfied, and $Y$ does not have descendants.  Let $\mathcal{G}$ be the DAG representing the relationships among $\mathbf{X}$ and $Y$. The $Y$'s parent is $\pa(Y)$ and $\mathbf{X}$ are non-descendant variable of $Y$. $\mathbf{X' = X}\backslash \{X_i\}$ and $\mathbf{x'}$ is a value of $\mathbf{X'}$. Let $\pa'(Y) = \pa(Y) \backslash \{X_i\}$, and $\mathbf{p'}$ be a value of $\pa'(Y)$ if $X_i \in \pa(Y)$. $y$ stands for $Y=1$ for brevity. 
	
	$\ece (X_i, Y \mid \mathbf{x'}) = P(y \mid do (X_i=1), do (\mathbf{X' = x'})) - P(y \mid do (X_i=0), do (\mathbf{X' = x'}))$. Let $x_i$ represent value 0 or 1. We show how $P(y \mid do (X_i=x_i), do (\mathbf{X' = x'}))$ is reduced to a normal probability expression.  
	
	Firstly, based on the Markov condition, $Y$ is independent of all its non-descendant nodes given its parents, and since $Y$ has no descendants, we have  for any $X_i \notin \pa(Y)$, $X_i \indep Y \mid \pa(Y) = \mathbf{p}$ and hence $P(y \mid do(X_i=x_i), do(\mathbf{X' = x'})) = P(y \mid do(X_i=x_i), do(\pa(Y) = \mathbf{p}))$ Therefore, $\ece (X_i, Y, \mathbf{x}) = 0$ since when the parents of $Y$ are set to certain values, all other variables are independent of $Y$ and their changes do not affect $Y$. 
	
	Secondly, let $X_i \in \pa(Y)$, we will prove $P(y \mid do(X_i=x_i),do(\pa'(Y)=\mathbf{x'}))=P(y \mid X_i = x_i,  \pa'(Y)=\mathbf{x'})$. This can be achieved by repeatedly using Rule 2 in Theorem~3.4.1~\cite{Pearl2009_Book}. Let $\pa(Y) = \{X_{j_1}, X_{j_2}, \ldots, X_{j_k}\}$.  Let $do (x)$ denote $do(X=x)$.
	
	\begin{equation*}
		\begin{aligned}
			&P(y \mid do(x_i),do(x_{j_1}),do(x_{j_2}),\ldots, do(x_{j_k}))  
			=P(y \mid x_i,do(x_{j_1}),do(x_{j_2}),\ldots, do(x_{j_k})) \\
			&since~Y \indep X_i | X_{j_1}, X_{j_2},\dots, X_{j_k}~in~\mathcal{G}_{\overline{X_{j_1}},\overline{X_{j_2}},\ldots,\overline{X_{j_k}},\underline{X_i}} \\
			&=P(y \mid x_i,x_{j_1},do(x_{j_2}),\ldots, do(x_{j_k}))~~~~since~Y \indep X_{j_1} | X_{i}, X_{j_2},\dots, X_{j_k}~in~\mathcal{G}_{\overline{X_{j_2}},\ldots,\overline{X_{j_k}},\underline{X_{j_1}}} \\
			&Repeat~(k-1)~times \\
			&=P(y \mid x_i,x_{j_1},x_{j_2},\ldots, x_{j_k}) =  P(y \mid x_i,\mathbf{p}') = P(y \mid x_i,\mathbf{x}') \\
		\end{aligned}
	\end{equation*}
	$\ece (X_i, Y \mid \mathbf{x'}) = P(y \mid do (X_i=1), do (\mathbf{X' = x'})) - P(y \mid do (X_i=0), do (\mathbf{X' = x'})) = P(y \mid X_i=1, \mathbf{p'}) - P(y \mid X_i=0, \mathbf{p'})$ if $ X_i \in \pa(Y)$.
	
	Hence, the theorem is proved. 
\end{proof}

The following proposition is for the Remark in lines 160-177.

\begin{proposition}[G-identifiability of $P(y \mid do(\pa(Y)))$ without an ordered assignment sequence]
	When there are no latent common causes of any pairs of variables in the system, $P(y \mid do(\pa(Y)))$ is G-identifiable \footnote{G-identifiability is defined in Theorem 4.4.1 (Pearl and Robins 1995) in~\cite{Pearl2009_Book}. We omit it in the Appendix for brevity and we state its condition in the proof. Simply speaking, identifiability means that the causal effect can be estimated from data that is faithful to the underlying causal graph.}  with any ordered sequence of value assignments for variables in $\pa(Y)$.  	
\end{proposition}	

\begin{proof}
	We use the same notations as in Theorem 4.4.6 in~\cite{Pearl2009_Book} so readers can easily see the G-identifiability condition in Theorem 4.4.6 is satisfied.
	
	Let $\pa(Y) = (X_{1}, X_{2}, \ldots, X_{n})$ be any ordered sequence of parents of $Y$. Let the value assignments be in the same order as that in sequence $\pa(Y)$, such that e.g. $do(X_{k-1})$ is done prior to $do(X_{k})$. Let $\mathbf{W_k}$ be a set of non-descendants of $(X_k, X_{k+1}, \ldots, X_n)$ and have either $Y$ or $X_k$ as descendant in graph $G_{\underline{X}_k, \overline{X}_{k+1}, \dots, \overline{X}_n}$. 
	
	The objective is to show the following independency holds for all $X_k$  in $\pa(Y)$. \\
	$(Y \indep X_k \mid X_1, \ldots, X_{k-1}, \mathbf{W_1, W_2,  \ldots, W_k})_{G_{\underline{X}_k, \overline{X}_{k+1}, \dots, \overline{X}_n}}$. \\
	Then $P(y \mid do(\pa(Y)))$ is G-identifiable (Theorem 4.4.6 in~\cite{Pearl2009_Book}). 
	
	Since $\pa(Y)$ includes all parents of $Y$ and there are no latent common causes of any pairs of variables, all paths from variables in $\mathbf{W_1, W_2,  \ldots, W_k}$ to $Y$ are $d$-separated by $\pa(Y)$ in graph $G_{\underline{X}_k, \overline{X}_{k+1}, \dots, \overline{X}_n}$. Therefore, we only need to show the following independency holds.
	
	$(Y \indep X_k \mid X_1, \ldots, X_{k-1})_{G_{\underline{X}_k, \overline{X}_{k+1}, \dots, \overline{X}_n}}$. 
	
	$\{X_{1}, X_{2}, \ldots, X_{n}\}$ are all parents of $Y$. When the outgoing edges of $X_k$ are removed (i.e. $\underline{X}_k$), the front door links from $X_k$ to $Y$ are removed.  When the incoming edges of $X_{k+1}, \dots, {X_n}$ are removed (i.e. $\overline{X}_{k+1}, \dots, \overline{X}_n$), the backdoor paths of $(X_k, Y)$ cannot go through them and have to go through $X_1, \dots, X_{k-1}$. Note that, $X_1, \dots$, or $X_{k-1}$ cannot be a collider in a path into $Y$ since the edge direction is into $Y$. All the possible backdoor paths of $(X_k, Y)$  are $d$-separated by $X_1, \dots, {X_{k-1}}$ in graph $G_{\underline{X}_k, \overline{X}_{k+1}, \dots, \overline{X}_n}$. Therefore, $(Y \indep X_k \mid X_1, \ldots, X_{k-1})_{G_{\underline{X}_k, \overline{X}_{k+1}, \dots, \overline{X}_n}}$.
	
	Since $\pa(Y) = (X_{1}, X_{2}, \ldots, X_{n})$ is in any order, $P(y \mid do(\pa(Y)))$ is G-identifiable with any ordered sequence of value assignments in $\pa(Y)$. 
\end{proof}

$P(y \mid do(\pa(Y)))$ is given in Theorem~\ref{theo:main} when the causal sufficiency is satisfied, i.e. there are no latent common causes of any pairs of variables in the system.

\section{Implementation}    
We now discuss the implementation of the ECEI Algorithm.

We first discuss the practical problems for finding combined direct causes and interactions and estimating $\eece$ with a large parent set.

\subsubsection{Dealing with combinatorial explosion}
The combined variables (including interactions) can be too many because of the combinatorial explosion. We will need to have the most informative combinations to make the explanations simple and effective.

A simple and sensible way is to set a minimum improvement threshold to prevent combined variables with only slight improvements in $\overline{\eece}$ over their component variables. We use the parameter  $\epsilon$ for this.     

We also need to remove redundant combined variables. If a combined variable has the same frequency as any of its component variables, this combined variable is redundant since its $\overline{\eece}$ is the same as that of its component variable. A closed frequent pattern set~\cite{wang-closet} will avoid generating the redundant combined variables. 

For combined variables, we have considered both values (0 and 1) of each component variable. For example, for a pair of variables $X_i$ and $X_j$, there are four possible combined variables with the values $(1, 1)$,  $(1, 0)$, $(0, 1)$, and $(0, 0)$ of their component variables. Whether using 0 of a component variable or not depends on if values 1 and 0 of the component variable are symmetric. For example, if $X_i$ represents (male, female), and $X_j$ denotes (professional and non-professional), 1 and 0 in both variables are symmetric and four combinations are meaningful. 
If $X_i$ and $X_j$ represent two words: 1 for presence and 0 for absence, then 0 provides much less information than 1 for understanding text and hence 0 and 1 are asymmetric. The three combinations involving 0 are not meaningful and 0 will not be considered in a combined variable. 
In the symmetric case, the four variables of four combined variables of their component variables are dependent. We call them combined variables with the same base variable set. We will use one which gives the largest $\overline{\eece}$ in the modelling to avoid too many derived dependent variables.      

\subsubsection{Dealing with data insufficiency when calculating conditional probabilities}

When calculating $\overline{\eece}$ in Definition14, 
we will need the conditional probability $P(y \mid X_i = x_i, \mathbf{P'=p'})$ where $\mathbf{P'}$ includes all variables in $\epa(Y)$ which are associated with $X_i$ but excludes those containing $X_i$ or a part of $X_i$ (when $X_i$ is an interaction or a combined direct cause). The set $\mathbf{P'}$ can be large and their values partition data into small subgroups, most of which do not contain enough samples to support reliable conditional probability estimation. In causal effect estimation, one common solution is to estimate $P(X_i=1 \mid \mathbf{P'=p'})$, called propensity scores~\cite{Rosenbaum1983_PropensityScore}, and use the propensity scores to partition data into subgroups for conditional probability estimation and bias control~\cite{Stuart2010_Matching}. Propensity score estimation needs building a regression model for each $X_i \in \epa(Y)$ and is time consuming. In the implementation, we use a subset of $\mathbf{P'}$ in which each variable has the highest association with $X_i$ as the conditional set. For estimating $\eece$, we also use a subset of $\mathbf{P'}$ for estimating conditional probabilities.

\subsubsection{Algorithm}
The proposed ECEI algorithm is listed in Algorithm~1. 

$\pc(Y)$ means Parents and Children of the target $Y$. In our problem setting, $Y$ does not have descendants, and hence, $\pc(Y)=\pa(Y)$.  Several algorithms have been developed for discovering $\pc(Y)$, such as PC-Select~\cite{PC-select-2010}, MMPC (Max-Min Parents and Children)~\cite{Tsamardinos2006_MMPC} and  HITON-PC~\cite{Aliferis2003_Hiton}.  These algorithms use the framework of constraint-based Bayesian network learning and employ conditional independence tests for finding the PC set of a variable. Their performance is very similar. 
We choose PC-Select in the R package \emph{PCalg}~\cite{kalisch2012causal}. This is shown in Line 1 of Algorithm 1. 

There are many packages available for frequent pattern mining. We use Apriori() in the R package \emph{arules} (https://github.com/mhahsler/arules). 

Other parts of the algorithm are self-explanatory or have been explained in the previous two subsections. 

The main costs of the algorithm are from PC$\_$Select and Apriori. Finding $\pa(Y)$ takes $O(|\mathbf{X}|^l)$ where $l$ is the size of the maximal conditional set for conditional independence tests, and usually $l$=3-6. The complexity of Apriori is around $O(|\mathbf{X'}|^k)$ where $\mathbf{X'}$ is the set of frequent items (variable values) and its size is at most $2|\mathbf{X}|$. $l$ is the maximum size of a combined variable, typically a small number 2-4 for good interpretation. So, the overall complexity is $O(|\mathbf{X}|^l)$. The algorithm may not be able to efficiently handle a large number of variables such as hundreds of variables and this is a limitation.        

\section{Experiment details}

\subsection{Data sets}

We conduct experiments on five data sets. The results from the  data sets are self-explanatory and can be evaluated by common sense. 


We have processed the Adult data set to make the explanation easy. The Adult data set (Table~\ref{tab_Adult_data}) was retrieved from the UCI Machine Learning Repository~\cite{Bache+Lichman:2013} and it is an extraction of the 1994 USA census database. It is a well known classification data set to predict whether a person earns over 50K or not in a year. We recoded the data set to make the causes for high/low income easily understandable. 

The two other data sets are downloaded from the same repository~\cite{Bache+Lichman:2013}  and binarised. Each value in a categorical attribute is converted to a binary variable where 1 indicates the presence of the value. A numerical attribute is binarised by the median. In the combined variables, the variables from the same attributes will not be considered as candidates.

\begin{table}
	\small
	\center
	\caption{Summary of the Adult data set}
	\label{tab_Adult_data}
	\begin{tabular}{|l|cc|c|}
		\hline
		Attributes & Yes & No & Comment \\
		\hline
		age $<$ 30 & 14515 & 34327 & young \\
		age $>$ 60 & 3606 & 45236 & old \\
		private & 33906 & 14936 & private company employer \\
		self-emp & 5557 & 43285 & self employment \\
		married & 22397 & 26463 & married incl. with spouse \\
		gov & 6549 & 42293 &  government employer \\
		education-num$>$12 & 12110 & 36732 & Bachelor or higher\\
		education-num$<$9 & 6408 & 42434 & education years\\
		Prof & 23874 & 24968 & professional occupation \\
		white & 41762 & 7080 & race \\
		male & 32650 & 16192 & \\
		hours $>$ 50 & 5435 & 43407 & weekly working hours \\
		hours $<$ 30 & 6151 & 42691 & weekly working hours \\
		US & 43832 & 5010 & nationality\\
		$>$50K & 11687 & 37155 & annual income, outcome\\
		\hline
	\end{tabular}
\end{table}

The Olivetti face dataset from AT\&T Laboratories Cambridge \footnote{https://scikit-learn.org/0.19/datasets/olivetti\_faces.html} is used for the evaluation with image data. Images are flattened to 1D vectors, and a random forest is used for classification. We also use LIME to convert an image for which we need to explain its predicted label, which is then input to ECEI. 

The 20 newsgroups dataset \footnote{https://scikit-learn.org/0.19/datasets/twenty\_newsgroups.html} is used to conduct the experiment with text data. Raw text data are first vectorised into TF-IDF features. Then, a random forest model is trained for classification using the TF-IDF dataset. We use LIME to convert a text for which we need to explain its predicted label to the interpretable representation. The representation is fed into ECEI to generate explanations. 

\end{document}